\documentclass[letterpaper]{article}
\usepackage{aaai23} 
\usepackage{times} 
\usepackage{helvet} 
\usepackage{courier} 
\usepackage[hyphens]{url} 
\usepackage{graphicx} 
\usepackage{graphicx} 
\usepackage{natbib} 
\usepackage{caption} 
\usepackage{color}
\usepackage{tabularx}
\usepackage{overpic}
\usepackage{comment}
\usepackage{amsthm}
\usepackage{amsfonts}
\usepackage{amsmath}
\usepackage{subcaption}

\newtheorem{theorem}{Theorem}[section]
\newtheorem{corollary}[theorem]{Corollary}
\newtheorem{lemma}[theorem]{Lemma}

\title{On Computing Makespan-Optimal Solutions for Generalized Sliding-Tile Puzzles}
\author{Marcus Gozon \qquad\quad Jingjin Yu}
\affiliations{
University of Michigan \qquad Rutgers University
}
\usepackage{xspace}

\def\gstp{\textsc{GSTP}\xspace}
\def\mogstp{\textbf{\textsc{MOGSTP}}\xspace}
\def\cfc{\textsc{CFC}\xspace}
\def\ttfsat{\textbf{\textsc{2/2/4-SAT}}\xspace}

\begin{document}
\maketitle
\begin{abstract}
In the $15$-puzzle game, $15$ labeled square tiles are reconfigured on a $4\times 4$ board through an escort, wherein each (time) step, a single tile neighboring it may slide into it, leaving the space previously occupied by the tile as the new escort. We study a generalized sliding-tile puzzle (\gstp) in which (1) there are $1+$ escorts and (2) multiple tiles can move synchronously in a single time step. Compared with popular discrete multi-agent/robot motion models, \gstp provides a more accurate model for a broad array of high-utility applications, including warehouse automation and autonomous garage parking, but is less studied due to the more involved tile interactions.
In this work, we analyze optimal \gstp solution structures, establishing that computing makespan-optimal solutions for \gstp is NP-complete and developing polynomial time algorithms yielding makespans approximating the minimum with expected/high probability constant factors, assuming randomized start and goal configurations.
\end{abstract}

\section{Introduction}\label{sec:intro}
The $15$-puzzle \cite{loyd1959mathematical} is a sliding-tile puzzle in which fifteen interlocked square tiles, labeled $1$-$15$, and an empty escort square are located on a $4\times 4$ square game board (see Fig.~\ref{fig:15-gstp}). In each time step, a tile neighboring the escort may slide into it, leaving an empty square that becomes the new escort. The game's goal is to reconfigure the tiles to realize a row-major ordering of the labeled tiles. We study a natural generalization of the $15$-puzzle, in which the game board is an arbitrarily large rectangular grid with $1+$ escorts.
In addition, tiles can move synchronously in a given time step assuming no collision under uniform movement.
We call this problem the \emph{generalized sliding-tile puzzle} or \gstp. 

\begin{figure}
    \centering
    \begin{overpic}
    [width=0.60\columnwidth]{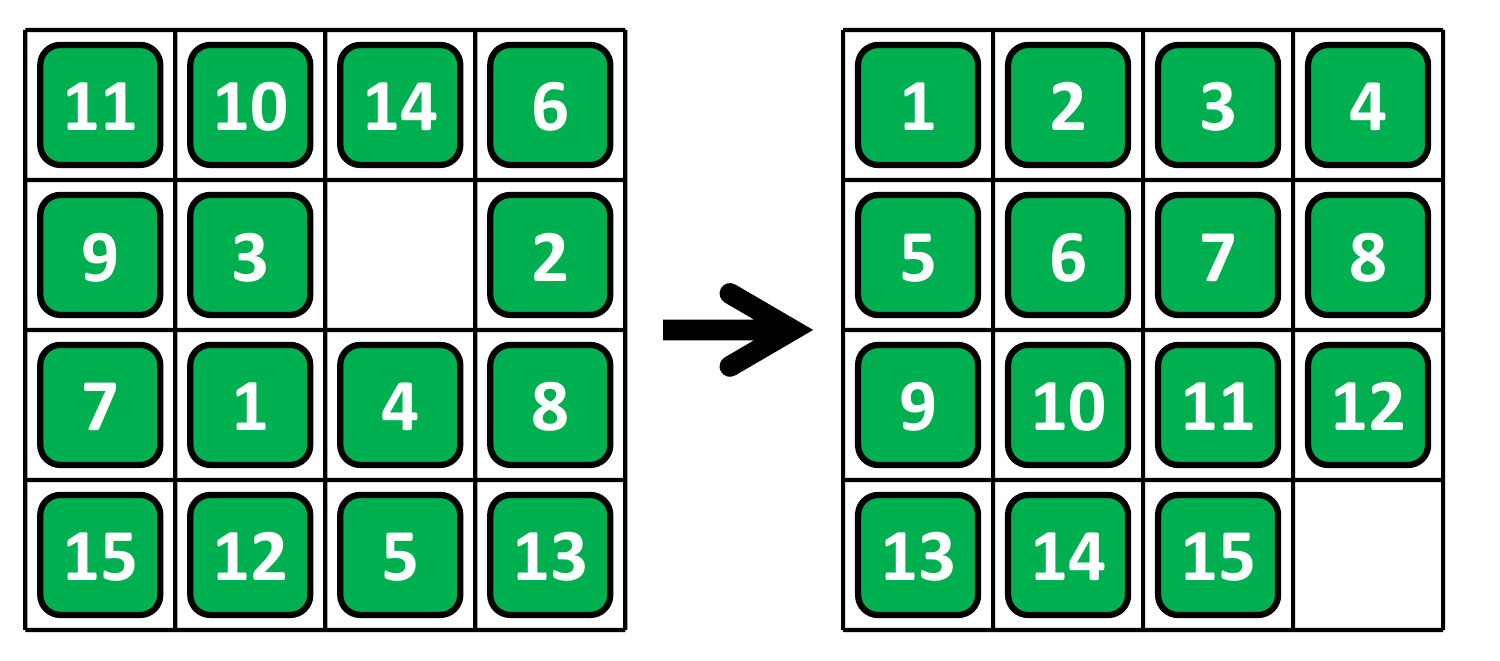}
    \end{overpic}
    \caption{Start and goal configurations of a $15$-puzzle instance. In  \gstp, there can be $1+$ escorts and multiple tiles may move synchronously, e.g., tile $3$ and $9$ may move to the right in a single step in the left configuration.}
    \label{fig:15-gstp}
\end{figure}

\gstp provides a high-fidelity discretized model for multi-robot applications operating in grid-like environments, including the efficient coordination of a large number of robots in warehouses for order fulfillment \cite{wurman2008coordinating,mason2019developing}, motion planning in autonomous parking garages \cite{guo2023efficient}, and so on. 
A particularly important feature of \gstp is that, given two neighboring tiles sharing a side, one tile may only move in the direction toward the second tile if the second tile moves in the same direction. Otherwise, if the second tile moves in a perpendicular direction, a collision occurs, which we call the \emph{corner following constraint} or \cfc. 
Consideration of \cfc renders \gstp different from popular multi-agent/robot pathfinding (MAPF) problems \cite{stern2019multi} in which a classical formulation allows the second tile to move in a direction perpendicular to the moving direction of the first tile. Ignoring \cfc significantly reduces the steps required to solve a tile reconfiguration problem, making computing optimal solutions less challenging, but is less accurate in modeling many real-world applications.  

Given the strong connections between \gstp and today's grid-based multi-robot applications seeking ever more optimal solutions, we must have a firm grasp on the fundamental optimality structure of \gstp. Towards achieving such an understanding, this work studies the induced optimality structure in computing makespan-optimal solutions for \gstp, and brings forth the following main contributions: 
\begin{itemize}
    \item We establish that computing makespan-optimal solutions for \gstp is NP-complete with or without an enclosing grid. The problem remains NP-complete when there are $\lfloor |G|^\epsilon \rfloor$ escorts, where $|G|$ is the grid size (i.e., the total number of grid cells) and $0 < \epsilon < 1$ is a constant. 
    \item We establish tighter makespan lower bounds for \gstp for all possible numbers of escorts. On an $m_1 \times m_2$ grid with $k$ escorts, in expectation, solving \gstp requires $\Omega(\frac{m_1 m_2}{k})$ steps for $1 \le k < \min(m_1, m_2)$ and $\Omega(m_1 + m_2)$ steps for $k \ge \min(m_1, m_2)$. 
    \item We establish tighter makespan upper bounds for \gstp for all possible numbers of escorts that match the corresponding makespan lower bounds, asymptotically, thus closing the makespan optimality gap for \gstp. This leverages a key intermediate result showing that \gstp instances on $2\times m$ and $3 \times m$ grids can be solved in $O(m)$ steps. For all upper bounds, via careful analysis, we further provide a constant factor that is relatively low, considering \cfc's severe restrictions on tile movements.  
\end{itemize}

Some proofs are sketched or omitted; see supplementary materials for additional details. 

\section{Related Work}
Modern studies on MAPF and related problems originated from the investigation of the generalization of the $15$-puzzle \cite{loyd1959mathematical} to the ($N^2-1$)-puzzle, with work addressing both computational complexity \cite{RATNER1990111} and the computation of optimal solutions \cite{culberson1994efficiently,culberson1998pattern}. Gradually, graph-theoretic abstractions emerged that introduced non-grid-based environments and allowed more escorts (i.e., there can be more than one empty vertex on the underlying graph). 
Whereas such problems are solvable in polynomial time if only a feasible solution is desired \cite{kornhauser1984coordinating,auletta1999linear,yu2013linear}, computing optimal solutions are generally NP-hard \cite{wilson1974graph,goldreich2011finding,surynek2010optimization,yu2013structure,Demaine2019CMP}.
With the graph-based generalization, \cfc is generally not enforced as the geometric constraint lengthens a motion plan and complicates the reasoning. 

Due to its close relevance to a great many high-impact applications, e.g., game AI \cite{pottinger1999implementing}, warehouse automation \cite{wurman2008coordinating,mason2019developing}, great interests started to develop in quickly computing (near-)optimal solutions for MAPF \cite{silver2005cooperative}. With this development, a variant of the ($N^2-1$)-puzzle was introduced, which does not require the presence of escorts \cite{standley2010finding}. In other words, in the most well-studied MAPF formulation, any non-self-intersecting chain of agents may potentially move synchronously, one following another, in a single step. In \cite{standley2010finding}, a bi-level algorithmic solution framework, \emph{operator decomposition} (OD) $+$ \emph{independence detection} (ID), is built upon the general idea of \emph{decoupling} \cite{erdmann1987multiple}, which treats each agent individually as if other agents do not exist and handles agent-agent interactions on demand. 
A super-majority of modern MAPF methods have generally adopted a bi-level decoupling search approach. Representative work along this line includes \emph{increasing cost-tree search} (ICTS) \cite{sharon2013increasing}, 
\emph{conflict-based search} (CBS) and variants \cite{sharon2015conflict,barer2014suboptimal,li2021eecbs}, \emph{priority inheritance with backtracking} (PIBT) \cite{okumura2022priority}, and most recently, \emph{lazy constraints addition search
for MAPF} (LaCAM) \cite{okumura2023lacam}. Besides search-driven methods, reduction-based approaches have also been proposed \cite{surynek2012towards,erdem2013general,yu2016optimal}. 

In contrast, MAPF formulations similar to \gstp, i.e., considering \cfc, have received relatively muted attention. On the side of computational complexity, besides the hardness result of the $(N^2-1)$-puzzle \cite{RATNER1990111} and a recent followup \cite{DEMAINE201880}, it has been shown that computing total distance-optimal solutions with \cfc is NP-complete in environments with specially crafted obstacles \cite{geft2022refined}. We note that sliding-tile puzzles can easily become PSPACE-hard in non-grid-based settings \cite{hopcroft1984complexity}, even for unlabeled tiles \cite{solovey2015hardness}. While of practical importance, hardness for computing optimal solutions for \gstp in obstacle-free settings has not been established. 
On the side of computational efforts in addressing \gstp, \cfc has been studied partially as part of $k$-robustness \cite{atzmon2018robust}. A recent SoCG competition has been held \cite{fekete2022computing} that addresses exactly the \gstp problem but with a focus on computing solutions for a set of benchmark problems. A variation of \gstp was studied in \cite{guo2023efficient} targeting autonomous parking garage applications. These computational studies largely leave unanswered fundamental questions on \gstp, including computational complexity and optimality bounds.

\section{Preliminaries}\label{sec:prelim}
\subsection{The Generalized Sliding-Tile Puzzle}
In the \emph{generalized sliding-tile puzzle} (\gstp), on a rectangular $m_1 \times m_2$ grid $G=(V, E)$ lies $n < m_1m_2$ tiles, uniquely labeled $1, \ldots, n$. 
A \emph{configuration} of the tiles is an injective mapping from $\{1, \ldots, n\} \to V = \{(v_y, v_x)\}$ where $1 \le v_y \le m_1$ and $1 \le v_x \le m_2$. 
Tiles must be reconfigured from a random configuration $\mathcal S=\{s_1, \ldots, s_n\}$ to some goal configuration $\mathcal G=\{g_1, \ldots, g_n\}$, usually a row-major ordering of the tiles, subject to certain constraints. 
Specifically, let the \emph{path} of tile $i$, $1\le i \le n$, be $p_i : \mathbb{N}_0 \to V$, and so \gstp seeks a \emph{feasible path set} $P = \{p_1, \ldots, p_n\}$ such that the following constraints are met for all $1\le i, j\le n$, $i \ne j$ and $\forall t \geq 0$: 
\begin{itemize}
    \item Continuous uniform motion: $p_i(t+1) = p_i(t)$ or $(p_i(t+1), p_i(t)) \in E$,
    \item Completion: $p_i(0) = s_i$ and $p_i(T) = g_i$ for some $T\ge 0$,
    \item No meet collision: $p_i(t) \neq p_j(t)$,
    \item No head-on collision: $(p_i(t)=p_j(t+1) \land p_i(t+1) = p_j(t)) = false$,
    \item Corner-following constraint: let $e_i(t) = p_i(t+1) - p_i(t)$ be the movement direction vector. If $p_i(t+1) = p_j(t)$, then $e_i(t) \not\perp e_j(t)$.
\end{itemize}

Let $T_P$ be the smallest $T\ge 0$ such that the completion constraint is met for a given path set $P$. Naturally, it is desirable to compute $P$ with minimum $T_P$. We define the decision version of makespan-optimal \gstp as follows. 

\vspace{2mm}
\noindent\mogstp\\
\noindent INSTANCE: A \gstp instance and a positive integer $K$.\\
\noindent QUESTION: Is there a feasible path set $P$ with $T_P \le K$?

\subsection{\ttfsat}\label{sec:224}
We will need a specialized SAT instance called \ttfsat for our hardness result, defined as follows. 

\vspace{2mm}
\noindent \ttfsat\\
\noindent \textbf{INSTANCE}: A boolean satisfiability instance with $n$ variables $x_1, \ldots x_n$ and $n$ clauses $c_1, \ldots c_n$. Each clause $c_j$ has $4$ literals, and each variable $x_i$ appears across all clauses exactly $4$ times in total, twice negated and twice unnegated.

\noindent \textbf{QUESTION}: Is there an assignment to $x_1, \ldots, x_n$ such that each clause $c_i$ has exactly two true literals?
\vspace{2mm}

\ttfsat was shown to be NP-complete in \cite{RATNER1990111}, which is subsequently employed to show the hardness of the $(N^2-1)$-puzzle.

\subsection{Feasibility and Known Makespan Bounds}
It is well-known that the $(N^2-1)$-puzzle may not always have a solution \cite{loyd1959mathematical} due to the configurations forming two connected graphs. More formally, it can be shown that the configurations of an $(N^2-1)$-puzzle are partitioned into two groups, each of which is isomorphic to the \emph{alternating group} $A_{N^2-1}$ \cite{wilson1974graph}. Because moves on a \gstp instance on an $N \times N$ grid with a single escort can be ``slowed down'' to equivalent moves on an $(N^2-1)$-puzzle, they share the same feasibility. The same remains true for rectangular grids. Checking feasibility can be performed in linear time \cite{wilson1974graph}. On the other hand, also clear from \cite{wilson1974graph}, when there are two or more escorts, a \gstp instance is always feasible. To summarize,
\begin{lemma}
\gstp with a single escort may be infeasible. The feasibility of \gstp with a single escort can be checked in linear time. \gstp with two or more escorts is feasible. 
\end{lemma}

Given a feasible $(N^2-1)$-puzzle, each tile can be moved to its goal in $O(N)$ steps since a tile is within $O(N)$ distance to its goal and $O(1)$ steps are needed to switch two tiles. This suggests an $O(N^3)$ algorithm, which readily extends to an $O(m_1m_2\max(m_1,m_2))$ step algorithm on an $m_1\times m_2$ grid. This is also an upper bound for \gstp with a single escort. \gstp with more escorts is studied in the context of automated garages \cite{guo2023efficient}, with results on $\Theta(m_1m_2)$ escorts and $(2m_1 + 2m_2 - 4)$ escorts. To summarize, the following is known: 
\vspace{2mm}

\newcolumntype{Y}{>{\centering\arraybackslash}X}
\noindent
\begin{small}
\begin{tabularx}{\columnwidth}{| c | Y |}
\hline
Number of escorts  &  Makespan upper bound \\ \hline
$1$ & $O(m_1m_2\max(m_1,m_2))$  \\ \hline
$(2m_1+2m_2-4)$ & $O(m_1m_2)$ \\ \hline
$\Theta(m_1m_2)$ & $O(\max(m_1,m_2))$ \\ \hline
\end{tabularx}
\end{small}

It is easy to see that $\Omega(\max(m_1, m_2))$ is a makespan lower bound in expectation. It can be shown that the makespan lower bound is close to $(m_1 + m_2)$ with high probability when there are $\Omega(m_1m_2)$ tiles \cite{guo2022sub15}.



\subsection{The Rubik Table Algorithm}\label{sec:rta}
A notable tool, Rubik tables \cite{szegedy2023rubik}, has been applied to derive polynomial-time, $1.x$-optimal solutions to classical MAPF problems on grids \cite{guo2022sub15}. This tool will also be employed in this work. We will use the following theorem with an associated algorithm. 

\begin{theorem}[Rubik Table Algorithm for 2D Grids \cite{szegedy2023rubik}] Let an $m_1\times m_2$ grid be filled with tiles labeled $1, \ldots, m_1m_2$. A row (resp., column) shuffle can arbitrarily permute a row (resp., column) of tiles. Then, the tiles can be rearranged from any configuration to the row-major configuration using $m_1$ row shuffles, followed by $m_2$ column shuffles, and then $m_1$ row shuffles. Alternatively, the tiles can be rearranged using $m_2$ column shuffles, followed by $m_1$ row shuffles, and then another $m_2$ column shuffles. 
\end{theorem}

Fig.~\ref{fig:rta} illustrates running the Rubik table algorithm over a $4\times 3$ grid, using a row-column-row shuffle sequence. 

\begin{figure}[h]
    \centering
    \includegraphics[width=\columnwidth]{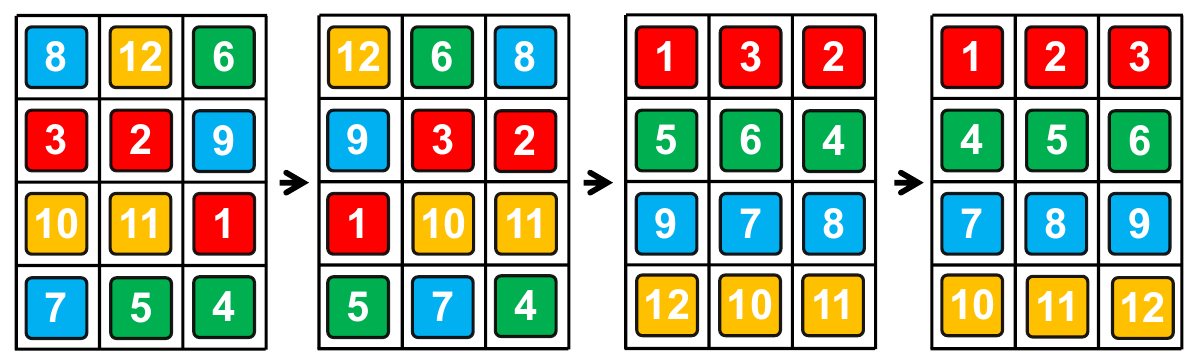}
    \caption{Applying the Rubik table algorithm to rearrange tiles on a $4\times 3$ grid using a sequence of row shuffles, followed by column shuffles, followed by row shuffles.}
    \label{fig:rta}
\end{figure}

\section{Intractability of \mogstp}\label{sec:complexity}
We proceed in this section to establish the NP-completeness of \mogstp on square grids, which will show  

\begin{theorem}\label{t:nph}
\mogstp is NP-complete, with or without an enclosing grid. 
\end{theorem}

First, we sketch the proof to provide key ideas behind the reduction of hardness. Then, detailed constructions of the required gadgets and the full instance construction follow. 

\subsection{Proof Outline}
We prove via a reduction from \ttfsat \cite{RATNER1990111} defined in Sec.~\ref{sec:224}. 
Our reduction constructs an \mogstp instance to force a flow of literal tiles from variable gadgets to clause gadgets in matching pairs, forming a truth side of literals and a false side of literals (realized through a \emph{gadget train}, see Fig.~\ref{fig:nph-sketch}(a) for a sketch and explanation). For each variable $x_i$, $1\le i\le n$, there are four sliding tiles labeled $x_i^1,x_i^2,\bar x_i^1,\bar x_i^2$ that correspond to the four literals for $x_i$, the first pair positive and the second pair negative.
When the context is clear, we simply say \emph{literals} instead of \emph{literal tiles}. A variable gadget (see Fig.~\ref{fig:nph-sketch}(b) and Fig.~\ref{fig:vg}) is constructed that forces the pair of unnegated literals (e.g., $x_i^1$ and $x_i^2$, ``+'' tiles in the figure) to only exit together from one side of the gadget (e.g., left) while forcing the pair of negated literal (e.g., $\bar x_i^1$ and $\bar x_i^2$, , ``-'' tiles in the figure) to exit together from the opposite side, each passing through limited openings of the \emph{rails} that flank the train and move in the opposite direction.
After all $4n$ literals exit from the $n$ variable gadgets, there are $2n$ each on the left and right side of the rails. These literals are then routed into clause gadgets (see Fig.~\ref{fig:nph-sketch}(c) and Fig.~\ref{fig:cg}), each allowing at most two literals to enter from each side.
The overall \mogstp instance is constructed such that if the \ttfsat is satisfiable, then in the \mogstp, the $2n$ literal tiles that move to the left side of the train can be chosen to be the true literals in the given truth assignment, and so all literal tiles can then be readily routed to the clause gadgets. Similarly, in the other direction of the reduction, because exactly $n$ pairs of literal tiles must be on the left side in a makespan-optimal solution, the corresponding $2n$ literals can be set to positive to satisfy the \ttfsat instance. 
\begin{figure}[h]
    \centering
    \begin{overpic}
    [width=\columnwidth,tics=5]{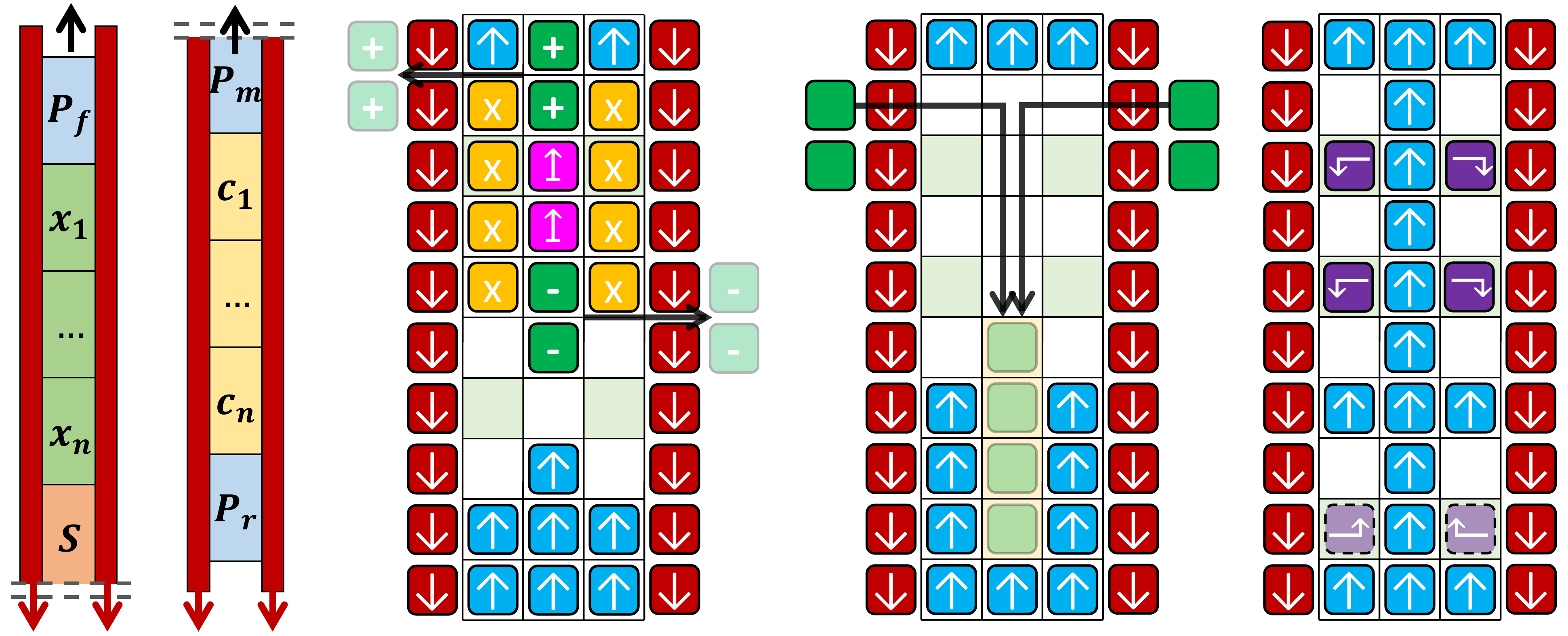}
    \put(7.8, -3){{\small (a)}}
    \put(33, -3){{\small (b)}}
    \put(62.5, -3){{\small (c)}}
    \put(88.5, -3){{\small (d)}}
    \end{overpic}
    \caption{Pieces of \mogstp. (a) Sketch of the train-like \mogstp instance split into two halves. The upward-moving gadget train is surrounded by two (red) rails of tiles that move strictly downwards, with a few gaps (not shown here, see Fig.~\ref{fig:rails}) to allow tiles to exit/enter.  
    The train, from top to bottom, contains a front padding car $P_f$, variable cars $x_1, \ldots, x_n$, a security car $S$, a middle padding car $P_m$, clause cars $c_1, \ldots, c_n$, and a rear padding car $P_r$. (b) A variable gadget (center $10\times 3$ portion) is constructed to force unnegated (``+'')  and negated (``-'') literal tiles to exit from different sides. The exited titles will be outside the rails. (c) A clause gadget is constructed to allow at most two literals to come in from each side of the rails. (d) The security car where the upper four purple tiles will exit to block variable exits on the rails (see Fig.~\ref{fig:rails}). The lower two light purple blocks are goals for two tiles initially on the rails (Fig.~\ref{fig:rails}).} 
    \label{fig:nph-sketch}
\end{figure}
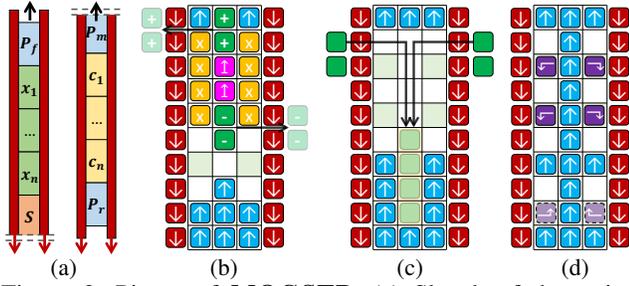

\subsection{Gadgets}
Our gadgets consist of \emph{preset tiles} that move in a fixed direction throughout the solution routing process. The \emph{up} (resp., \emph{down}) tiles move one step up (resp., down) at each time step, which can be forced by setting their goals a distance upwards (resp. downwards) equal to the given makespan of the \mogstp instance.

\subsubsection{Rail (Gadget)} A \mogstp instance contains two symmetric rails (red strips on the two sides in Fig.~\ref{fig:nph-sketch}, with more details in Fig.~\ref{fig:rails}) consisting of down tiles with gaps, which are $3 \times 1$ blocks of escorts.
Each rail contains three gaps, separated into two groups: two lower gaps are designated as \emph{variable exits}, and a single upper gap functions as a \emph{clause entrance}.
A down tile separates the variable exits. The four purple tiles from the security car will enter the middle of these gaps and then move with the rails until the end. There are two (purple) tiles initially in the entrance gaps that will later enter the security car. These gaps will be explained in more detail. 

\begin{figure}[h]
    \centering
    \begin{overpic}
    [width=\columnwidth]{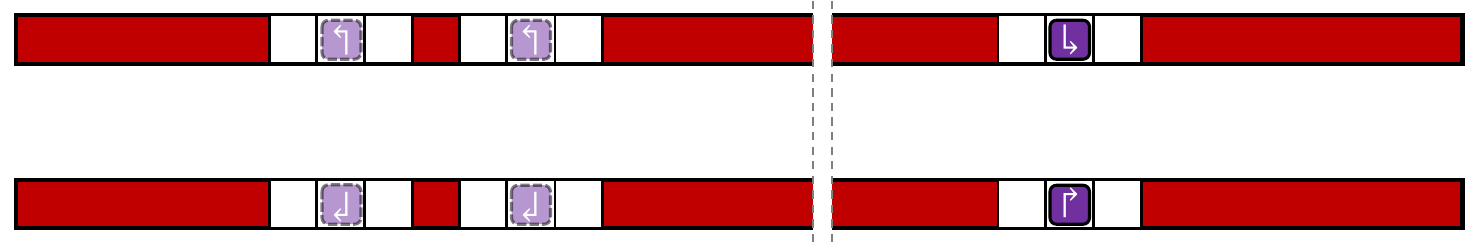}    
    \put(17, 7.5){{\small variable exit gaps }}
    \put(57, 7.5){{\small clause entrance gaps }}
    \end{overpic}
    \caption{Part of the rails, rotated $90$ degrees clockwise from Fig.~\ref{fig:nph-sketch}(a), showing the variable exit and clause entrance gaps.}
    \label{fig:rails}
\end{figure}

\subsubsection{Variable Car (Gadget)} 
A variable car ($x_1, \ldots, x_n$ blocks in Fig.~\ref{fig:nph-sketch}) is an upwards moving $10 \times 3$ block whose start configuration is shown in Fig.~\ref{fig:nph-sketch}(b). For the variable car corresponding to $x_i$, besides the (blue) up tiles as marked,  there are two unnegated literal tiles (the two ``+" tiles) corresponding to $x_i^1$ and $x_i^2$, and two negated literal tiles (the two ``-" tiles) corresponding to $\bar x_i^1$ and $\bar x_i^2$. These tiles must be moved to some clause cars to be introduced shortly. 
There are also (pink) \emph{single-delay} up tiles that must pause in place exactly once throughout the execution of the \mogstp instance. 
Additionally, there are eight obstacle tiles (the x tiles) whose goal configurations are three spots lower within the same variable car. These obstacle tiles help ensure that the pairs of positive and negative literals split up onto different sides.

\begin{figure}[h]
    \centering
    \includegraphics[width=\columnwidth]{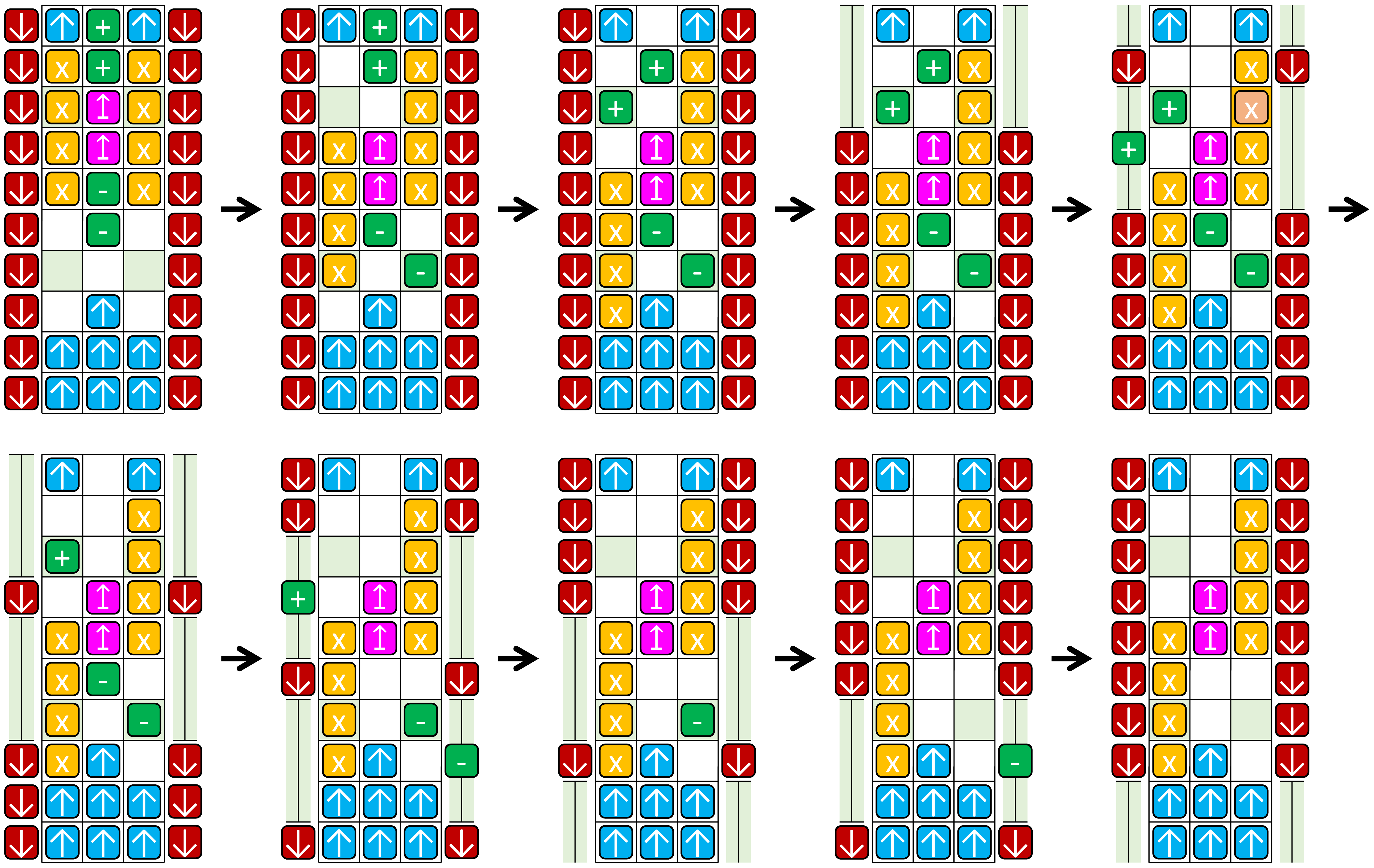}
    \caption{Illustration of how (green) literal tiles may exit a variable gadget in pairs. The bottom left subfigure shows the four lower gaps on the rails, each a $3 \times 1$ block.}
    \label{fig:vg}
\end{figure}

\begin{lemma}
    As a variable car passes by the variable exits on the rails, the positive and negative literals can only exit to different sides of the rails.
\end{lemma}
\begin{proof}[Proof Sketch]
Only literal and obstacle tiles may move outside a variable car (the $10\times 3$ grid). It can be shown that obstacle tiles should not change columns.
Because of this, unnegated (resp., negated) literals can only exit from the $3$th (resp., $7$th) row. This forces the obstacle tiles to become asymmetric on the two sides of a variable car, resulting in the unnegated literals exiting from one side of the car and the negated literals exiting from the opposite side. One such exit sequence is illustrated in Fig.~\ref{fig:vg}. 
\end{proof}

\subsubsection{Clause Car (Gadget)}
As shown in Fig.~\ref{fig:nph-sketch}(c), the clause car is an upward-moving $10 \times 3$ subgrid entirely composed of up tiles and requires $4$ literal tiles corresponding to $c_i$ in the goal configuration. With two symmetric $3\times 1$ gaps on the rail, it is clear that at most two literals can enter from each side, as shown in Fig.~\ref{fig:cg}.
\begin{figure}[h]
    \centering
    \includegraphics[width=\columnwidth]{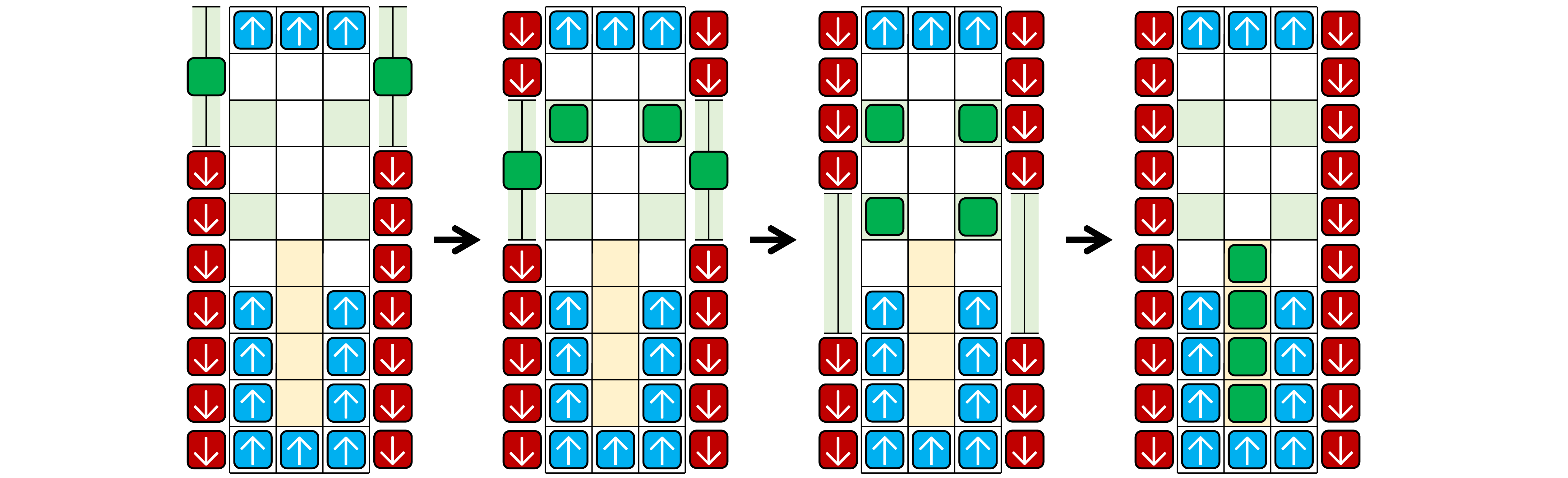}
    \caption{Green literal tiles entering a clause car gadget.}
    \label{fig:cg}
\end{figure}

\subsubsection{Security Car (Gadget)}
Shown in Fig.~\ref{fig:nph-sketch}(d), the security car is an upward-moving $10 \times 3$ subgrid whose frame consists of up tiles with four additional purple tiles at rows $3$ and $5$ that need to be injected into the middle of the four variable exits on the rails (see Fig.~\ref{fig:rails} for reference) as they pass by. This prevents these gaps from being used by clause gadgets. It will receive the two purple tiles that are initially inside the clause entrances (see Fig.~\ref{fig:rails}) to go to row $8$.

\subsection{Complete Specification and Reductions}
We construct the \mogstp instance as follows. Let $d = 24n + 38$; our upwards moving train is a $d \times 3$ block; from top to bottom, the three padding cards have $4$, $4n$, and $24$ rows of up tiles, respectively. The middle padding car allows the literal tiles to reorder before entering the clause cars. Initially, the bottoms of the rails are aligned with the bottom of the gadget train. 
The \emph{variable exit} openings occupy rows $d+2$ to $d+4$ and $d+6$ to $d+8$ from the bottom. The \emph{clause entrance} opening occupy rows $d+14n+12$ to $d+14n + 14$.
While a grid is not needed, we can select our grid $G$ to be of size $4d \times 4d$ with the construct positioned in the middle horizontally, with the bottom of the construct starting at the $(d+1)$th row from the bottom of $G$.
The makespan bound $K$ is set to $d$. The \mogstp instance is fully specified.

\begin{proof}[Proof of Thm.~\ref{t:nph}]
If the \ttfsat instance is satisfiable, we can select the positive (resp., negative) literals to exit to the left (resp., right) of the rails from the variable cars, when they pass by the literal exit gaps. Then, these literals can reorder and enter into the clause gadgets as required, reaching the target goal configuration with a makespan of $d$. 

Similarly, in the other direction, if the \mogstp instance has a solution with a makespan of $d$, then the up/down tiles must move uninterrupted. In this case, four literals must exit a variable car in pairs of the same truth value to different sides. Subsequently, these literals reorder and enter the clause cars as described. Therefore, we can pick literal tiles on one side of the rails, e.g., left, and make their corresponding literals positive, ensuring all clauses are true. This yields a satisfying assignment for the \ttfsat instance. 

\mogstp is in NP since the existence of a solution can be readily checked, and a feasible solution can be computed in polynomial time similar to how $(N^2-1)$-puzzles are solved. Thus, \mogstp is NP-complete.
\end{proof}

$\mogstp$ remains NP-hard when we specify that there are exactly $\lfloor |G|^\epsilon \rfloor$, $0< \epsilon < 1$ escorts (where $|G|$ is the number of cells of the grid) by blowing up the grid by a polynomial amount and filling the extra space with stationary tiles to achieve the desired number of escorts. Then note that \ttfsat is still simulated through the movement of the literal tiles around the preset tiles, and in addition, a solution routing can be constructed in the same manner from a truth assignment.

\section{Tighter Makespan Lower \& Upper Bounds}\label{sec:bounds}
In this section, we first establish a tighter makespan lower bound as a function of the number of escorts. Then, we proceed to the more involved efforts of deriving tighter makespan upper bounds again as a function of the available number of escorts. The new and tighter lower and upper bounds are summarized in the table below. We further provide an exact constant for all upper bounds as a more precise characterization. In all cases, our upper and lower bounds match asymptotically, eliminating the gaps left by previous studies on \gstp. All upper bounds come with low-polynomial-time algorithms for computing the actual plan, which is clear from the corresponding proofs. 

\begin{table}[h!]
\begin{small}
\begin{tabularx}{\columnwidth}{| c | Y |}
\hline
$k$, the number of escorts  &  Makespan \textbf{lower} bound \\ \hline
$k < \min(m_1, m_2)$ & exp. $\Omega(\dfrac{m_1 m_2}{k})$ \\ \hline
$k \ge \min(m_1, m_2)$ & h.p. $\Omega(\max(m_1, m_2))$ \\ \hline  \hline
$k$, the number of escorts  &  Makespan \textbf{upper} bound \\ \hline
$k = 1$ & $(81 + o(1))m_1m_2$ \\ \hline
$k = 2$ & $(18 + o(1))m_1m_2$ \\ \hline
$2 < k < \min(m_1, m_2)$ & $(22 + o(1))\dfrac{m_1m_2}{\lfloor k/2\rfloor}$ \\ \hline
$k \geq m_1 + m_2 - 1$ & $34\max(m_1, m_2)$ \\ \hline
\end{tabularx}
\end{small}
    \caption{Our matching makespan lower and upper bounds.}
    \label{tab:bounds}
\end{table}
For convenience, instead of viewing \gstp through batched tile movements, we focus on the movement of the escorts, which encodes tile motion more concisely.
A straight contiguous train of tiles moving in a single step may be equivalently viewed as a \emph{jump} of an escort. 
Since the new escort position must remain in the same row or column, we call the jump a \emph{row jump} or \emph{column jump}, respectively.
In addition, we use \emph{rectangular shift} or \emph{r-shift}, as a fundamental motion primitive in which the escort cycles through the four corners of a rectangle, thus shifting all boundary elements by one tile in the opposite direction. 
We call the rectangular shift \emph{cwr-shift} (resp., \emph{ccwr-shift}) if the escort  traverses the corners in the counterclockwise (resp., clockwise) direction. 

\subsection{Tighter Makespan Lower Bounds}
Using escort jumps instead of tile moves lets us see immediately that a single time step can only change the sum of the Manhattan distances by $k \max(m_1, m_2)$, where $k$ is the number of escorts. The observation readily leads to a tighter makespan lower bound than the previously established $\Omega(m_1+m_2)$ (or $\Omega(\max(m_1,m_2))$).

\begin{lemma}
    The expected minimum makespan for \gstp on $m_1 \times m_2$ grids with $k$ escorts is $\Omega(\frac{m_1 m_2}{k})$.
\end{lemma}

\begin{proof}
Consider the sum of Manhattan distances $S$ of each tile's start and goal positions.
Over all possible start and goal configurations, each tile is expected to have a Manhattan distance of $\Omega(m_1 + m_2) = \Omega(\max(m_1, m_2))$ \cite{santalo2004integral}. 
Because there are $m_1m_2 - k$ tiles, we have $S = \Omega((m_1m_2 - k)\max(m_1, m_2))$, in expectation.
Because each of the $k$ escorts can jump a distance of $\max(m_1, m_2)$ within the same row/column, altering the Manhattan distance contribution by $1$ for each tile in its jump path, in a single time step, $S$ can only change by at most $k\max(m_1, m_2)$.
Thus, at least $\frac{S}{k\max(m_1, m_2)}= \Omega(\frac{m_1 m_2}{k})$ steps are needed to solve the instance in expectation. 
\end{proof}
Combining with the previously known lower bounds yields a tighter makespan lower bound of $\Omega(\frac{m_1 m_2}{k})$ for $k \leq \min(m_1, m_2)$ and $\Omega(\max(m_1, m_2))$ for $k \geq \min(m_1, m_2)$, which match our developed upper bounds, established next. 

\subsection{Tighter Makespan Upper Bounds: Outline}
We leverage RTA (Sec.~\ref{sec:rta}) to establish tighter makespan upper bounds for \gstp. Each round of row/column shuffles in RTA can be executed in parallel, potentially leading to a significantly reduced makespan. However, the shuffles do not readily translate to feasible sliding-tile motion; performing in-place permutation of tiles in a single row/column is impossible. 
To enable the application of RTA, instead working with one grid row/column, we simulate row/column shuffles by grouping multiple rows or columns together. 
Therefore, at a high level, we derive better upper bounds by:
\begin{itemize}
    \item Applying RTA to obtain three batches of row or column shuffles (see, e.g., Fig.~\ref{fig:rta}) with escorts treated as labeled tiles. Each batch of shuffles will be executed to completion (via simulations according to rules of \gstp) before the next batch is started. 
    \item In a given batch of row/column shuffles, adjacent rows/columns will be grouped together (e.g., two or three rows per group), on which tile-sliding motions will be planned to realize the desired shuffles. 
\end{itemize}

Performing efficient tile-sliding motions with the \cfc constraint is key to establishing tighter upper bounds. We first describe subroutines for solving \gstp with $1$ or $2$ escorts on $3\times m$ and $2\times m$ grids. These subroutines will then be used to solve general \gstp instances. 

%

\subsection{Upper Bounds for 2-3 Rows with 1-2 Escorts}
Our \gstp algorithms will build on subroutines for sorting multiple rows. We first prove such a routine on $3 \times m$ grids. 

\begin{lemma}\label{l:3m}
Feasible \gstp instances with a single escort on a $3 \times m$ grid can be solved in $120m$ steps.
\end{lemma}
\begin{proof}
We give a procedure that sorts the right $\frac{1}{3}$ of the $3 \times m$ grid in $O(m)$ steps. A recursive application of the procedure then yields an overall $O(m +\frac{2}{3}m +\frac{4}{9}m + \ldots) = O(m)$ makespan. 

To start, we move the escort to the bottom left corner for both the start and goal configurations, which takes $4$ steps. These will be the new start/goal configurations. 
From here, for tiles on a $3\times m$ grid, let $B$ denote the set of tiles corresponding to the $\lfloor \frac{m-2}{3} \rfloor$ rightmost columns in the goal configuration. We refer to these tiles as $B$ tiles and the rest as $W$ tiles. 
We will treat the boundary cells as a \emph{circular highway} moving clockwise and the inner middle line as a \emph{workspace} to move $B$ tiles to their destination. As an example, the $B$ (resp., $W$) tiles are shown in dark gray (resp., light gray) in Fig.~\ref{fig:3m}(b)-(g). The algorithm operates in three stages: (1) move $B$ tiles to the highway, (2) arrange $B$ tiles properly in the workspace, and (3) move $B$ tiles to goals.

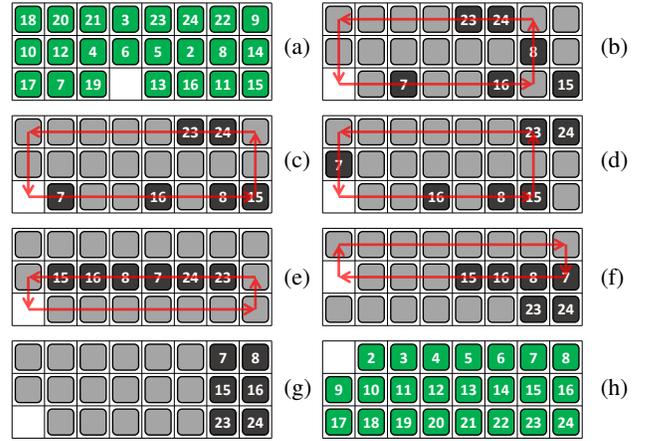
\begin{figure}[h]
    \centering
    \begin{overpic}
    [width=\columnwidth]{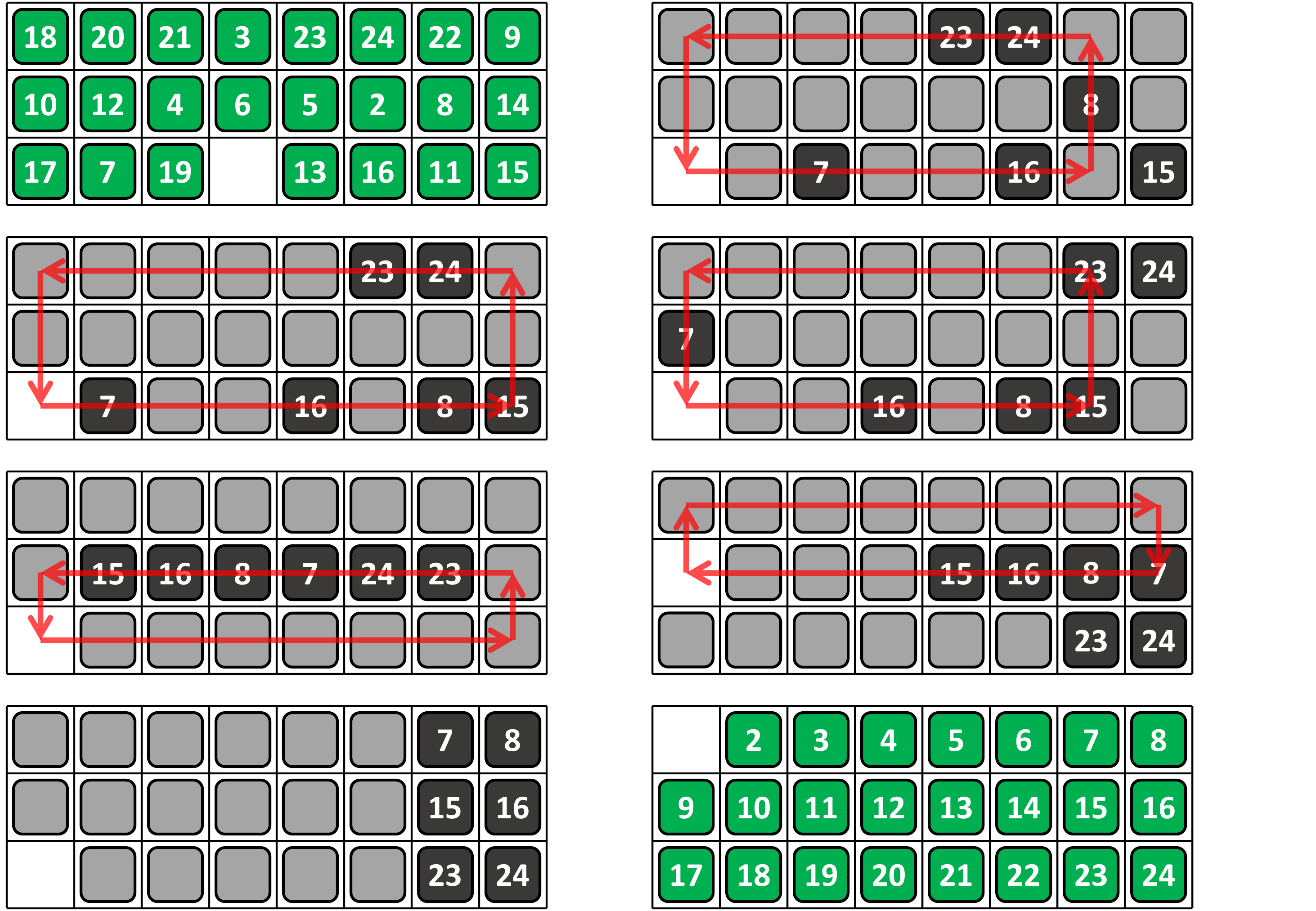}    
    \put(43.5, 61){{\small (a)}}
    \put(93.5, 61){{\small (b)}}
    \put(43.5, 43){{\small (c)}}
    \put(93.5, 43){{\small (d)}}
    \put(43.5, 24.5){{\small (e)}}
    \put(93.5, 24.5){{\small (f)}}
    \put(43.5, 7){{\small (g)}}
    \put(93.5, 7){{\small (h)}}
    \end{overpic}
    \caption{Sorting right $\frac{1}{3}$ on a $3 \times 8$ grid with one escort. (a) and (h) are the start and goal configurations. (b)$\to$(c): A cwr-shift inserts $B$ tile $8$ to the circular highway. (c)$\to$(d)$\to\ldots\to$(e): A series of r-shifts orders $B$ tiles in the workspace. (e)$\to$(g): Additional r-shifts move $B$ tiles to goals.}
    \label{fig:3m}
\end{figure}

To execute the first stage, if a $B$ tile in the workspace has a $W$ tile above it, then execute a cwr-shift to insert the leftmost such $B$ tile into the highway to not affect tiles to the right (Fig.~\ref{fig:3m}(b)-(c)). 
Otherwise, apply \emph{adjustment} cwr-shifts to the circular highway until a $B$ tile in the workspace has a $W$ tile above it.
Because there are $m-2$ $B$ tiles, at most $m-2$ adjustments are needed to move a $W$ tile over each $B$ tile, and so the total number of steps for this stage is at most $4[(m-2) + (m-2)] = 8m - 16$.

The second stage uses the same operation to insert the $B$ tiles into the workspace. The difference is that $B$ tiles are now being inserted in the exact spot in the workspace corresponding to the desired permutation.
Through the process, a tile in $B$ never makes a full lap around the circular highway. Therefore, at most $2m+1$ adjustments are needed, with at most $m-2$ $B$ tile insertions, taking at most $4[(2m+1) + (m-2)] = 12m - 4$ steps.

In the third stage, apply r-shifts to move $B$ tiles to their goals as shown in Fig.~\ref{fig:3m}(e)-(g), taking $4[m - \lfloor \frac{m-2}{3} \rfloor]$ steps.

Now, approximately the right third of the grid has been solved in $4[6m - 5 - \lfloor \frac{m-2}{3} \rfloor]$ steps; we recurse in the same manner for $m \geq 5$ and solve the base case of $m = 4$ in 53 times steps \cite{korf2008linear} by treating the problem as a normal $(n^2 - 1)$-puzzle instance.
Through careful counting, we can conclude that $120m$ steps are always sufficient. 
\end{proof}

The $120m$ makespan can be significantly reduced with more careful analysis, which we omit due to limited space. The important takeaway is Lemma~\ref{l:3m} shows \gstp on $3\times m$ grids can be solved in $O(m)$ steps, sufficient for establishing the upper bounds in our claimed contribution. In what follows, we describe related results needed to get the constant factors stated in Table~\ref{tab:bounds} omitting the proofs. 

If we have two escorts, we can cycle them on opposite corners of their respective r-shifts to allow two cwr-shifts to happen simultaneously, leading to the following. 
\begin{corollary}
\gstp instances with two escorts on a $3 \times m$ grid can be solved in $60m$ steps.
\end{corollary}

With significant additional efforts but following a similar line of reasoning, we can establish on $2\times m$ grids that
\begin{lemma}\label{l:2m1}
Feasible \gstp instances with a single escort on a $2 \times m$ grid can be solved in $58m$ steps.
\end{lemma}

While simulating two row or column permutations at once can be useful in solving \gstp faster, the limited amount of space may prevent us from doing so. Instead, simulating the permutation of one row or column will be much more useful.

\begin{corollary}\label{c:2m11}
Given a single escort, a $2 \times m$ grid can be permuted to fill one of its rows arbitrarily in $27m$ steps.
\end{corollary}

With two escorts, we get significantly faster algorithms.
\begin{lemma}\label{l:2m2}
\gstp instances with two escorts on a $2 \times m$ grid can be solved in $10m - 13$ steps.
\end{lemma}

\begin{corollary}\label{c:2m21}
Given two escorts on the left of the top row of a $2 \times m$ grid, the bottom row can be arbitrarily permuted in $6m-1$ time steps, maintaining the position of the escorts.
\end{corollary}

Corollaries~\ref{c:2m11} and~\ref{c:2m21} will be instrumental in parallelizing row and column permutations necessitated by the RTA Shuffles without wasting additional steps in permuting the other row. 

\subsection{Tighter Makespan Upper Bounds for \gstp}
We are now ready to tackle solving full \gstp instances. 
For \gstp, we will only examine the case in which grid dimensions are at least $2$; the problem is otherwise trivial. 

\begin{theorem}
Feasible single-escort \gstp instances can be solved in $81m_1 m_2 + 6m_1 + 9m_2 - 3$ steps.
\end{theorem}
\begin{proof}
First, move the escort to the top left for start/goal configurations to get new start/goal configurations. 
Then, RTA is applied in a row-column-row fashion to yield three batches of row/column shuffles. Each batch requires sorting $m_1$ or $m_2$ rows or columns. In the $4 \times 4$ grid shown in Fig.~\ref{fig:gstp1}(a), a batch of row shuffles must permute each of the four rows highlighted in different colors. We are done if we can successfully perform each batch of shuffles. 

To execute a batch of shuffles, e.g., performing the four row shuffles on the $4 \times 4$ grid shown in Fig.~\ref{fig:gstp1}(a), we move the escort to the top left of the bottom two rows and apply Corollary~\ref{c:2m11} sort the last row. The procedure is repeated with the escort moved one row above, until there are only two top rows, at which point Lemma~\ref{l:2m1} is invoked to arrange the two rows simultaneously. The top two rows may not be solved exactly because not all $(N^2-1)$-puzzles are solvable, but the issue will resolve on its own if the \gstp instance is solvable. Other shuffles are executed similarly.
\begin{figure}[h]
    \centering
    \begin{subfigure}{0.24\columnwidth}
        \includegraphics[width=\columnwidth]{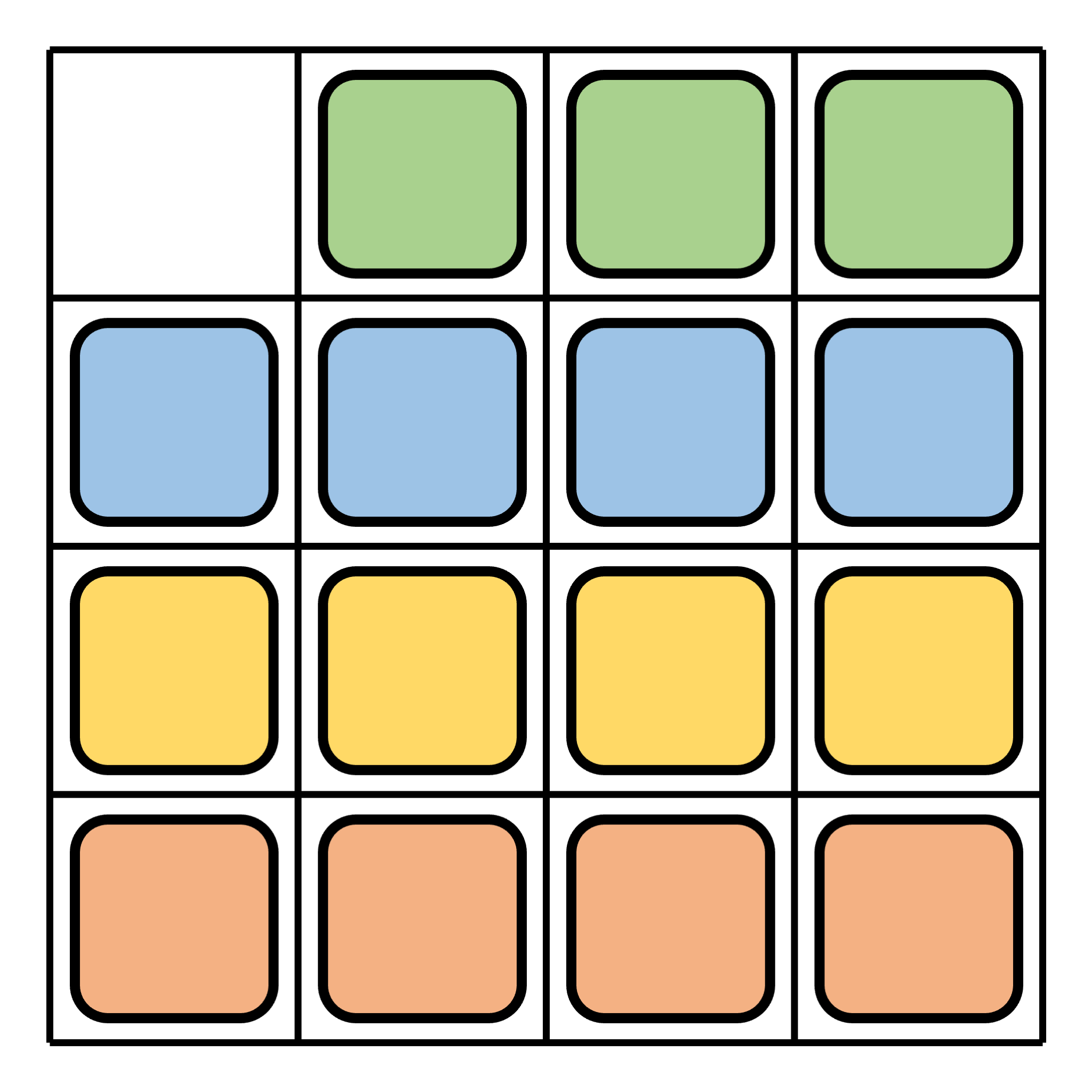}
        \caption{}
    \end{subfigure}
    \begin{subfigure}{0.24\columnwidth}
        \includegraphics[width=\columnwidth]{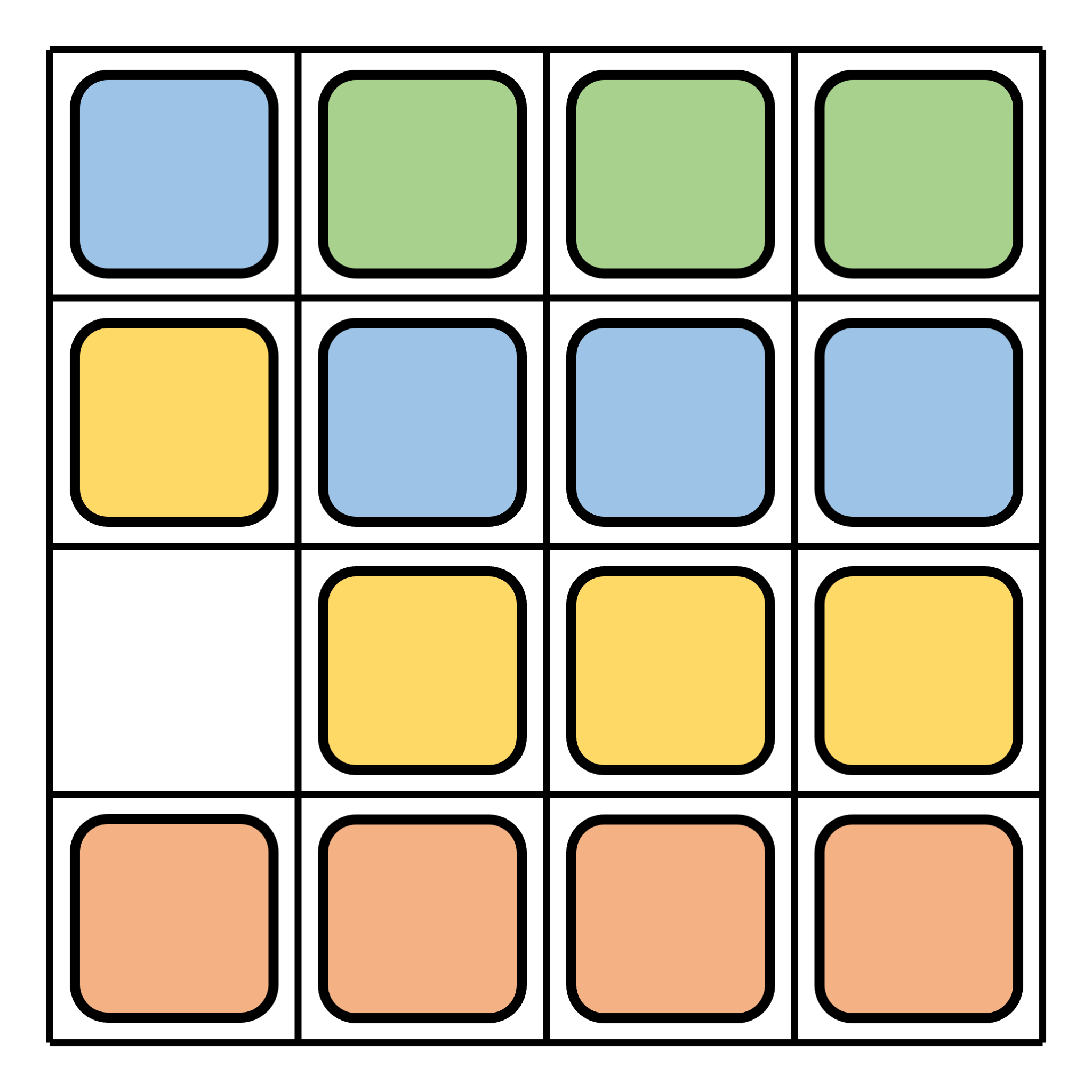}
        \caption{}
    \end{subfigure}
    \begin{subfigure}{0.24\columnwidth}
        \includegraphics[width=\columnwidth]{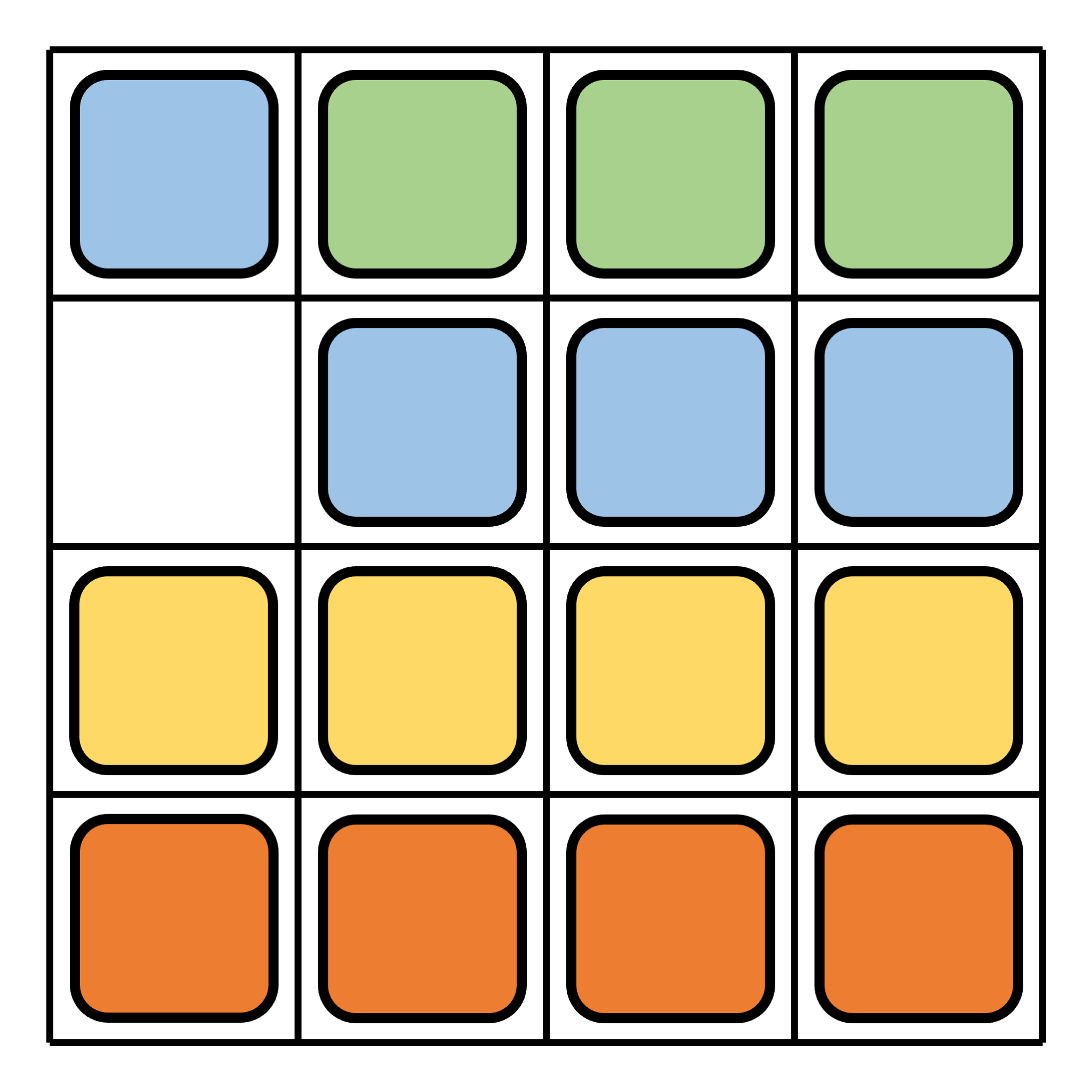}
        \caption{}
    \end{subfigure}
    \begin{subfigure}{0.24\columnwidth}
        \includegraphics[width=\columnwidth]{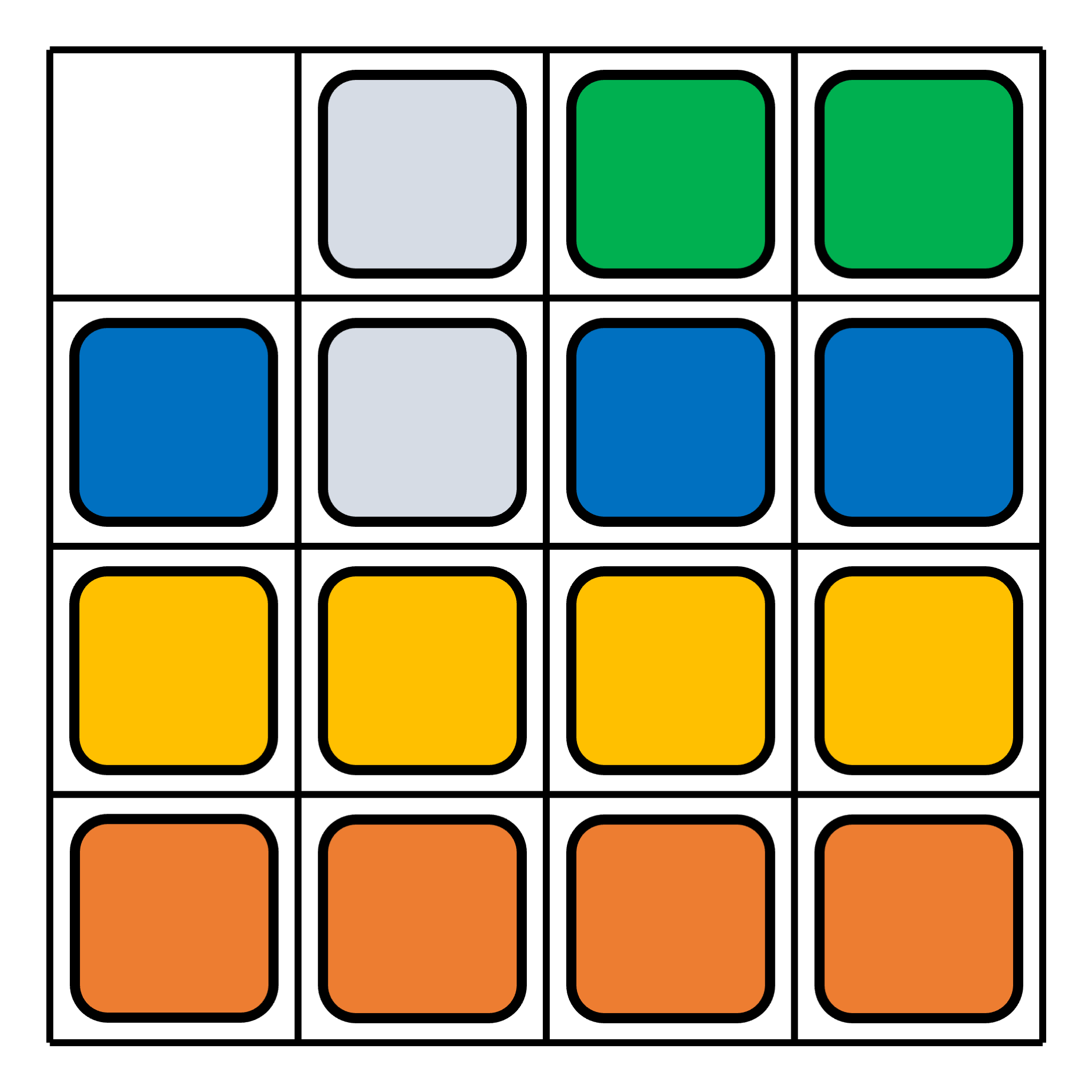}
        \caption{}
    \end{subfigure}
    
    \caption{Illustrating performing a batch of row shuffles on a $4 \times 4$ grid with a single escort. 
    (a). The (updated) start configuration, in which each row must be permuted. (b). To prepare for running Corollary~\ref{c:2m11}, the escort is moved to the top left of the last two rows. (c). After applying Corollary~\ref{c:2m11} to sort the last row, the escort is shifted above for the next application. (d) The top two rows will be sorted using Lemma~\ref{l:2m1}. Note that the top (resp., left) two rows (resp., columns) may not be fully solvable in the first two batches of shuffles, which is fine for the next set of column shuffles.}
    \label{fig:gstp1}
\end{figure}

Counting all steps, the total number is at most $81m_1 m_2 + 6m_1 + 9m_2 - 3$.
\end{proof}

\begin{theorem}
A two-escort \gstp instance can be solved in $18m_1 m_2 - 4m_1 - 5m_2 - 29$ steps.
\end{theorem}
\begin{proof}[Proof Sketch]
The proof is similar to the single escort case; with two escorts, we invoke Lemma~\ref{l:2m2} and Corollary~\ref{c:2m21} to speed up the process. The entire instance can be solved in $2[(m_1 - 2)(6m_2 - 1) + 10m_2 - 13] + [(m_2 - 2)(6m_1 - 1) + 10m_1 - 13] + 4 = 18m_1 m_2 - 4m_1 - 5m_2 - 29$ steps.
\end{proof}

\begin{theorem}
A \gstp instance containing $2 \leq k < \min(m_1, m_2)$ escorts, where $k$ is even, can be solved with a makespan less than $\frac{44m_1 m_2}{k} + m_1(5 - \frac{24}{k}) + 15m_2 - 29$.
\end{theorem}
\begin{proof}[Proof Sketch]
The main strategy is distributing the escorts across the rows/columns to introduce parallelism in solving a batch of row/column shuffles. For example, given $k=2\ell$ escorts, to solve a batch of $m_1$ row shuffles, we can distribute two escorts per $\frac{m_1}{\ell}$ rows. For each such $\frac{m_1}{\ell}$ rows, we invoke Lemma~\ref{l:2m2} and Corollary~\ref{c:2m21} to solve them, in parallel.  
This allows the entire batch of row shuffles to be completed in $O(\frac{m_1}{\ell})O(m_2)=O(\frac{m_1m_1}{\ell})=O(\frac{m_1m_1}{k})$ steps. 
Tallying over the three phases, the total number of steps required is bounded by
$\frac{44m_1 m_2}{k} + m_1(5 - \frac{24}{k}) + 15m_2 - 29$.
\end{proof}
Note that if $k$ is odd, we can simply ignore one escort.
It is also clear the results continue to apply for $k \min(m_1, m_2)$ by ignoring extra escorts, but we can get additional speedups when $k \geq m_1 + m_2 - 1$ due to enough room to use \ref{l:2m2} straightforwardly by having escorts along the top row and left column.
Compiling everything so far yields the claimed bounds given in Table.~\ref{tab:bounds}.


\section{Conclusion and Discussion}\label{sec:conclusion}
We show that it is NP-complete to compute makespan-optimal solutions for the generalized sliding-tile puzzle (\gstp). We further establish matching asymptotic makespan lower and upper bounds for \gstp for all possible numbers of escorts, and provide concrete constants for all makespan upper bounds. In ongoing and future work, we are examining (1) computing optimal solutions for other objectives for \gstp, (2) related variations of the \gstp formulation, and (3) developing practical algorithms for computing different optimal solutions for large-scale \gstp instances. 

\section*{Acknowledgement}\label{sec:ack}
We thank the reviewers and editorial staff for their insightful suggestions. This work is supported in part by the DIMACS REU program NSF CNS-2150186, NSF award CCF-1934924, NSF award IIS-1845888, and an Amazon Research Award. 
\bibliography{refs}
\end{document}